\documentclass[preprint,12pt]{elsarticle}

\usepackage{amssymb}
\usepackage{mathrsfs}
\usepackage{mathrsfs}
\usepackage{mathrsfs}
\usepackage{amsfonts}
\usepackage{amsfonts}
\usepackage{amsfonts}
\usepackage{mathrsfs}
\usepackage{amsmath}
\usepackage{graphicx}
\usepackage{multirow}
\usepackage{caption2}
\usepackage{float}
\usepackage{booktabs}
\usepackage{algorithm}
\usepackage{algorithmic}
\usepackage{subfigure}
\newtheorem{theorem}{Theorem}[section]

\numberwithin{equation}{section}

\newtheorem{corollary}[theorem]{Corollary}

\newtheorem{lemma}[theorem]{Lemma}

\newtheorem{remark}[theorem]{Remark}

\newenvironment{proof}[1][Proof]{\textbf{#1. }}{\ \rule{0.5em}{0.5em}}%
\journal{}

\begin{document}
%\graphicspath{{eps/}}
%\DeclareGraphicsExtensions{.eps}
\begin{frontmatter}



\title{Learning and approximation  capabilities of orthogonal super greedy algorithm \tnoteref{t1}} \tnotetext[t1]{The research was
supported by the National 973 Programming (2013CB329404), the Key
Program of National Natural Science Foundation of China (Grant No.
11131006).}

\author{Jian Fang$^1$  }

\author{Shaobo Lin$^{2}$\corref{*}}\cortext[*]{Corresponding author: sblin1983@gmail.com}

\author{Zongben Xu$^1$}

\address{1. Institute for Information and
System Sciences, School of Mathematics and Statistics, Xi'an
Jiaotong University, Xi'an, 710049, China

2. College of Mathematics and Information Science, Wenzhou
University, Wenzhou 325035, China }

\begin{abstract}
We consider the approximation capability  of orthogonal super greedy
algorithms (OSGA) and its applications in supervised learning. OSGA
is concerned with selecting more than one atoms in each iteration
step, which, of course, greatly reduces the computational burden
when compared with the conventional
 orthogonal greedy algorithm (OGA). We prove that even for
function classes  that are not the convex hull of the dictionary,
OSGA does not degrade the approximation capability of OGA provided
the dictionary is incoherent. Based on this, we deduce a tight
generalization error bound for OSGA learning. Our results show that
in the realm of supervised learning, OSGA provides a possibility to
further reduce the computational burden  of OGA in the premise of
maintaining its prominent  generalization capability.
\end{abstract}

\end{frontmatter}
% Introduction
%%%%%%%%%%%%%%%%%%%%%%%%%%%%%%%%%%%%%%%%%%%%%%%%%%%%
\section{Introduction}

A greedy algorithm is a stepwise inference process that follows the
problem solving heuristic of making the locally optimal choice at
each stape with the hope of finding a global optimum. The use of
greedy algorithms in the context of nonlinear approximation
\cite{Barron2008} is very appealing since it greatly reduces the
computational burden when compared with standard model selection
using general dictionaries. This property triggers avid research
activities of greedy algorithms in signal processing
\cite{Dai2009,Kunis2008,Tropp2004}, inverse problem
\cite{Donoho2012,Tropp2010} and sparse approximation
\cite{Donoho2007,Temlyakov2011}.

Greedy learning, or more specifically, applying greedy algorithms to
tackle supervised learning problems, has been proved to possess
charming generalization capability with lower computational burden
than the widely used coefficient-based regularization methods
\cite{Barron2008}. From approximation to learning, greedy learning
can be usually formulated as a four-stage stepwise learning strategy
\cite{Xulin2014}.
 The first one is the ``dictionary-selection'' stage which  constructs a suitable set of basis functions.
 The second one is the ``greedy-definition'' stage that sets the measurement
  criterion to choose new atoms (or elements) from the dictionary  in each greedy step.
 The third one is   the ``iterative-rule'' stage that defines the
 estimator based on the selected ``greedy atoms'' and the estimator obtained in the previous greedy step.  The last one
is the ``stopping-criterion'' stage which focuses on how to
terminate the learning process.

Since greedy learning's inception in supervised learning
\cite{Friedman2001}, the aforementioned four stages were
comprehensively studied for various purposes. For the
``dictionary-selection'' stage,  Chen et al. \cite{Chen2013} and Lin
et al. \cite{Lin2013a} proposed that the kernel based dictionary is
a good choice for greedy learning. For the ``greedy-definition''
stage, Xu et al. \cite{Xulin2014} pointed out that the metric of
greedy-definition is not uniquely the greediest one. They provided a
threshold to discriminate whether or not a selection is greedy and
analyzed the feasibility of such a discrimination measurement. For
the ``iterative-rule stage'', Barron et al. \cite{Barron2008}
declared that both relaxed greedy iteration and orthogonal greedy
iteration can achieve a fast learning rate for greedy learning. For
the ``stopping-criterion'' stage, Barron et al. \cite{Barron2008}
provided an   $l^0$ complexity regularization strategy and Chen et
al. \cite{Chen2013} proposed an $l^1$ complexity constraint
strategy.  All these results showed that as a new learning scheme,
greedy learning deserves avid studying due to its stepwise learning
character \cite{Friedman2001}.

Although the importance of a single stage of greedy learning was
widely studied \cite{Barron2008,Chen2013,Lin2013a,Xu2014}, the
relationship between these stages and their composite effects for
learning also need classifying. In the recent work \cite{Xulin2014},
Xu et al. established a relationship between the
``greedy-definition'' and ``stopping-criterion'' stages and
successfully  reduced the computational cost of greedy learning
without sacrificing the generalization capability. This implies that
the study of these relationships may bring additional benefits of
greedy learning. In this paper, we aim to study the relationship
between the ``dictionary-selection'' and ``greedy-definition''
stages of orthogonal greedy algorithms (OGA). Our idea mainly stems
from an interesting observation. We observe that if the selected
dictionary is an orthogonal basis, then it is not necessary to
define greedy learning as a stepwise strategy. Indeed, due to the
orthogonal property, we can select all required atoms from the
dictionary simultaneously. Conversely, if the dictionary is
redundant (or linear dependent), then greedy learning must be
defined as a stepwise strategy due to the redundant property which
usually causes a disorder of the learning process.
   This implies that
 specific features of a dictionary can
be used to modify the greedy definition.

% As a widely used greedy-type algorithm,
%orthogonal greedy algorithm (OGA) was proved to  be a prominent
%learner \cite{Barron2008,Chen2013a,Xulin2014}.
Therefore, if the coherence, a specific feature of a dictionary, is
utilized to describe the dictionary, we can improve the performance
of OGA in the direction of either reducing the computational burden
or enhancing the generalization capability. In this paper, we study
the learning capability of
 orthogonal super greedy algorithm (OSGA) which was proposed by Liu and Temlyakov
 \cite{Liu2012}. OSGA
 selects more than one atoms from a    dictionary in each
iteration step and hence reduces the computational burden of OGA.
The aim of the present paper can be explained in two folds. The first one is to study
the approximation capability of OSGA and the other is to pursue the
pros and cons of OSGA in the context of supervised learning.

For OSGA approximation, it was shown in \cite{Liu2012} (see also
\cite{Liu2012a}) that for incoherent dictionaries, OSGA  reduces the
computational burden   when compared with   OGA. It can be found in
\cite[Theorem 2]{Liu2012} that such a significant computational
burden-reduction  does not degrade the approximation capability if
the target functions belong to the convex hull of the dictionary.
However, such an assumption to the target functions  is very
stringent  if the dimension of variable is large \cite{Barron2008}.
Our purpose is to circumvent the above problem by deducing
convergence rates for functions not simply related to the convex
hull of the dictionary. Interestingly, we find that, even for
functions out of the convex hull of the dictionary, the
approximation capability of OSGA is similar as that of OGA
\cite{Barron2008}.

For OSGA learning, we find that if the dictionary is incoherent,
then OSGA learning with appropriate step-size can reduce the
computational burden of OGA learning further. In particular, using
the established approximation results of OSGA, we can deduce an
 almost  same learning rate as that of   OGA. This means that
 studying
the relationship between the ``dictionary-selection'' and
``greedy-definition'' stages can   build more efficient learning
schemes with the same rate of convergence as   OGA.

The paper is organized as follows. In Section 2, we review notations
and preliminary results in  greedy-type algorithms that are
frequently referred to throughout the paper. In Section 3,  we show
the main result of this paper where a general approximation theorem
for OSGA and  its applications in supervised learning are
established. In Section 4, we present a line of simulations to
verify our viewpoints. In Section 5, we   give   proofs of the main
results. In the last section, we further discuss the OSGA learning
and draw a simple conclusion of this paper.

\section{Greedy-type algorithms}
Let $H$ be a Hilbert space endowed with norm and inner product
$\|\cdot\|$ and $\langle\cdot,\cdot,\rangle$, respectively. Let
$\mathcal D=\{g\}_{g\in\mathcal D}$ be a given dictionary.  Define
$\mathcal L_1=\{f:f=\sum_{g\in D}a_gg\}.$ The norm of $\mathcal L_1$
is defined by $\|f\|_{\mathcal L_1}:=\inf\left\{\sum_{g\in \mathcal
D}|a_g|:f=\sum_{g\in \mathcal D}a_gg\right\}.$ We shall assume here
and later that the elements of the dictionary are normalized
according to $\|g\|=1$.

There exist several types of greedy algorithms \cite{Temlakov2008}.
The four most commonly used are the pure greedy, orthogonal greedy,
relaxed greedy and stepwise projection algorithms, which are often
denoted by their acronyms PGA, OGA, RGA and SPA, respectively. In
all the above greedy algorithms, we begin by setting $f_0:=0$. The
new approximation $f_k\; (k \ge 1)$ is defined based on $f_{k-1}$
and its residual $r_{k-1}:=f-f_{k-1}$. In OGA, $f_k$ is defined as
$$
                f_k=P_{V_k} f,
$$
where $P_{V_k}$ is the orthogonal projection onto
$V_k=\mbox{span}\{g_1,\dots,g_k\}$ and $g_k$ is defined as
$$
        g_k=\arg\max_{g\in\mathcal D}|\langle r_{k-1},g\rangle|.
$$

Let
$$
           M=M(\mathcal D)=\sup_{g\neq h,g,h\in\mathcal D}|\langle
           g,h\rangle|
$$
be the coherence   of the  dictionary $\mathcal D$. Let $s\geq 1$ be
a natural number. Initially, set $f^s_0=0$ and $r^s_0=f$, then the
OSGA  proposed in \cite{Liu2012} for each $k\geq 1$ can be
inductively define as the following.

1) $g_{(k-1)s+1},\dots,g_{ks}\in\mathcal D$ are chosen according to
$$
              \min_{i\in I_k}|\langle r^s_{k-1},g_i\rangle|\geq
              \sup_{g\in\mathcal D, g\neq g_i, i\in I_k}|\langle
              r^s_{k-1},g\rangle|,
$$
where  $I_k=[(k-1)s+1,ks]$.

2) Let $V_{ks}=\mbox{span}\{g_1,\dots,g_{ks}\}$ and define
\begin{equation}\label{OSGA estimator1}
           f^s_k:=P_{V_{ks}} f,
\end{equation}
and
$$
           r^s_k=f-f_k^s.
$$

 The following Lemma \ref{pure}  proved in \cite{Liu2012} shows
that OSGA can achieve the optimal approximation rate of $ks$ term
nonlinear approximation \cite{Temlyakov2003}.

\begin{lemma}\label{pure}
  Let $\mathcal D$ be a dictionary with coherence $M$. Then, for
  $s\leq (2M)^{-1}+1,$ the  OSGA estimator (\ref{OSGA estimator1})
  provides
   an approximation of $f\in\mathcal L_1$ with the
  following error bound:
$$
          \|r_k^s\|^2\leq 40.5\|f\|_{\mathcal L_1}(sk)^{-1}, \ k=1,2,\dots.
$$
\end{lemma}

\section{Approximation and learning by OSGA}
In this section, after presenting some basic conceptions of the
statistical learning theory, we deduce a general approximation
theorem concerning OSGA and pursue its applications in supervised
learning.

\subsection{Statistical learning theory}
   In most
of machine learning problems, data are taken from two sets: the
input space $X\subseteq \mathbf R^d$ and the output space
$Y\subseteq \mathbf R$. The relation between the variable $x\in X$
and the variable $y\in Y$ is not deterministic, and is described by
a probability distribution $\rho$ on $Z:=X\times Y$ that admits the
decomposition
$$
                    \rho(x,y)=\rho_X(x)\rho(y|x),
$$
in which $\rho(y|x)$ denotes the conditional (given $x$) probability
measure on $Y$, and $\rho_X(x)$ the marginal probability measure on
$X$. Let ${\bf z}=(x_i,y_i)_{i=1}^n$ be a set of finite random
samples of size $n$, $n\in\mathbf N$, drawn identically,
independently according to $\rho$ from $Z$. The set of examples
${\bf z}$ is called a training set. Without loss of generality, we
assume that $|y_i|\leq L$ for a prescribed (and fixed) $L>0$.

The  goal of supervised learning is to derive a function $f:X\rightarrow Y$ from a training set such that $f(x)$ is an effective and
reliable estimate of $y$ when $x$ is given.
 A natural measurement of the error
incurred by using $f(x)$ for this purpose is the generalization
error, given by
$$
                     \mathcal E(f):=\int_Z(f(x)-y)^2d\rho,
$$
which is minimized by the regression function \cite{Cucker2007},
defined by
$$
                     f_\rho(x):=\int_Yyd\rho(y|x).
$$
This ideal minimizer $f_\rho$ exists in theory only. In practice, we
do not know $\rho$, and  we can only access random examples from
$X\times Y$ sampled according to $\rho$.

Let $L^2_{\rho_{_X}}$ be the Hilbert space of $\rho_X$ square
integrable function on $X$, with norm denoted by $\|\cdot\|_\rho.$
With the assumption that $f_\rho\in L^2_{\rho_{_X}}$, it is well
known \cite{Cucker2001} that, for every $f\in L^2_{\rho_X}$, there
holds
\begin{equation}\label{equality}
                     \mathcal E(f)-\mathcal E(f_\rho)=\|f-f_\rho\|^2_\rho.
\end{equation} The task of the least square regression
problem is then to construct functions $f_{\bf z}$ that approximates
$f_\rho$, in the norm $\|\cdot\|_\rho$, using finite samples.

\subsection{Approximation capability of OSGA}

A consensus in the nonlinear approximation community is that
greedy-type algorithms can break the ``curse of dimensionality''
\cite{Temlyakov2003}. Lemma \ref{pure} seems to verify this
assertion, since a dimensional independent convergence rate was
deduced. We find, however, this is not exactly true since, in
practice, the condition that the target functions belong to the
convex hull of the dictionary becomes more and more stringent as the
dimension of variable grows \cite{Barron2008}.  The similar
phenomenon concerning OGA approximation was successfully tackled in
\cite{Barron2008} by proving convergence results for a variety of
function classes and not simply those are related to the convex hull
of the dictionary.

Along the flavor of \cite{Barron2008},
  we
study the behavior of  OSGA  approximation when the target functions
$f\in H$ are more general.  We consider the real interpolation
spaces \cite{Bennett1988}
$$
           \mathcal R_p=[H,\mathcal L_1]_{\theta,\infty},\
           0<\theta<1,
$$
with $p$ defined by
$$
           \frac1p=\theta+\frac{1-\theta}2=\frac{1+\theta}2.
$$
Recall that  $f\in[X,Y]_{\theta,\infty}$ if and only if for all
$t>0$, there holds \cite{Devore1993}
\begin{equation}\label{kfunctional}
               K(f,t)\leq Ct^\theta,
\end{equation}
where
$$
              K(f,t)=K(f,t,X,Y)=\inf_{h\in Y}\{\|f-h\|_X+t\|h\|_Y\}
$$
is the so-called $K$-functional.  The smallest $C$ satisfying
(\ref{kfunctional}) defines a norm for the interpolation space
$[X,Y]_{\theta,\infty}$.  Based on these preliminaries, we can
obtain the following Theorem \ref{theorem1}.

\begin{theorem}\label{theorem1}
Let $\mathcal D$ be a dictionary with coherence $M$. Then for all
$f\in H$, any $h\in\mathcal L_1$ and arbitrary $s\leq (2M)^{-1}+1$,
the   OSGA estimator (\ref{OSGA estimator1}) satisfies
\begin{equation}\label{main error}
      \|r_k^s\|^2\leq\|f-h\|^2+\frac{27}2\|h\|_{\mathcal
      L_1}^2(sk)^{-1},\ k=1,2,\dots,
\end{equation}
and therefore,
\begin{equation}\label{approx k functional}
       \|r_k^s\|\leq K\left(f,\frac{3\sqrt{6}}2(sk)^{-1/2}, H,\mathcal
       L_1\right), k=1,2,\dots.
\end{equation}
\end{theorem}

From the definition of the interpolation space $\mathcal R_p$ and
(\ref{approx k functional}), it follows that $f\in \mathcal R_p$
implies the rate of decay
$$
        \|r_k^s\|\leq C_1 (ks)^{-\theta/2}.
$$
A similar error estimate for OGA has been provided in \cite[Theorem
2.3]{Barron2008}, which says that for all $f\in H$ and any
$h\in\mathcal L_1$, the error of the OGA estimator satisfies
\begin{equation}\label{oga}
      \|r_N\|^2\leq\|f-h\|^2+4\|h\|_{\mathcal
      L_1}^2N^{-1},\ N=1,2,\dots.
\end{equation}
We note that OSGA  adds $s$ new atoms at each iteration and makes
one orthogonal projection at each iteration. After $k$ iterations of OSGA, there are totally
$sk$ atoms to build up the estimator. For comparison, OGA
adds one atom at each iteration and makes one orthogonal projection
at each iteration. While, it is obvious that
there need $sk$ iterations of OGA to deduce an estimator with $sk$
atoms. Thus, the computational cost of OSGA is near $s$ times lower
than OGA. (\ref{main error}) together with (\ref{oga})  yields that
such a computational burden reduction does not degenerate the
approximation capability. The reason of this is that the specific
feature of the dictionary, $M$-coherence, is used in OSGA. It can be
found in Theorem \ref{theorem1} that if $M>1/2$,  OSGA
coincides with OGA.

\subsection{OSGA learning}

It was pointed out in \cite{Liu2012} that OSGA can be applied in
compressed sensing very well. In this subsection, we pursue its
applications in supervised learning. It can   also be found in
\cite[Theorem 3.1]{Barron2008} that the error estimate formed as
(\ref{main error}) plays an important role in analyzing the
generalization capability of greedy-type algorithms. Based on this,
we can deduce the generalization error of OSGA   in the context of
regression.

Given training samples ${\bf z}$, we define the empirical norm and
inner product as
$$
              \|f\|_n^2=\frac1n\sum_{i=1}^n|f(x_i)|^2,
$$
and
$$
         \langle f,g\rangle_n=\frac1n\sum_{i=1}^nf(x_i)g(x_i),
$$
respectively. The OSGA learning scheme studied in this subsection is
shown in the following Algorithm 1.

%empirical risk as
%$$
%              \|f\|_n^2=\frac1n\sum_{i=1}^n|f(x_i)|^2,
%$$
%and
%$$
%           \mathcal E_{\bf
%           z}(f)=\|y-f\|_n^2=\frac1n\sum_{i=1}^n|y_i-f(x_i)|^2,
%$$

\begin{algorithm}[H]
\caption{ \textbf{\large{OSGA learning}}.} \label{algo1}

Initialization: Data ${\bf z}=(x_i,y_i)_{i=1}^n$, step-size $s$,
iteration number $m$, inner product $\langle\cdot,\cdot\rangle_n$, $
f_0=0$, $V_0=\varnothing$, $r^s_{{\bf z},0}(x)$ satisfies
$r_0(x_i)=y_i$, and $y(x)$ satisfies $y(x_i)=y_i$.

\underline{Stage1: Dictionary-selection}: Select
$$
   \mathcal D_N:=\left\{{g_i}: i=1, \dots, N \right\}
$$
 with $\|g_i\|_n=1$.

 \underline{Stage2: Greedy definition:} Choose
 $g_{(k-1)s+1},\dots,g_{ks}\in\mathcal D_N$   according to
\begin{equation}\label{select s atom}
              \min_{i\in I_k}|\langle r^s_{{\bf z},k-1},g_i\rangle_n|\geq
              \sup_{g\in\mathcal D, g\neq g_i, i\in I_k}|\langle
              r^s_{{\bf z},k-1},g\rangle_n|,
\end{equation}
where $r^s_{{\bf z},k-1}$ is the residual defined by  $r^s_{{\bf
z},k-1}:=y-f^s_{{\bf z},k-1}$ and $I_k=[(k-1)s+1,ks]$.

\underline{Stage3. Iterative rule}: Let
 $V_{ks} =Span(g_1,...,g_{ks})$. Compute the $k$ step approximation $f^s_{{\bf z},k}$ as:
\begin{equation}\label{iOGA}
  {f^s_{{\bf z},k}} = {P_{{\bf z},V_{ks}}}({  y})
\end{equation}
and the residual: $r^s_{{\bf z},k}:=y-f^s_{{\bf z},k}$,
 where $P_{{\bf z},V_{ks}}$ is the orthogonal projection  onto
space $V_{ks}$
in the metric of $\langle\cdot,\cdot\rangle_n$.\\

\underline{Stage4. Stopping criterion   }: If $k=m$ then stop and
obtain the final estimator $f^s_{{\bf z},m}$, otherwise set $k=k+1$
and repeat Stage 1-Stage 4.
\end{algorithm}

It is shown in Algorithm 1 that the only difference between OSGA and
OGA learning \cite{Barron2008} is that in OSGA there are $s$ atoms
selected in the ``greedy-definition'' stage.  Therefore   the
computational burden of OGA is further reduced. The first result in
this subsection is to illustrate that such a reduction do not
degrade the generalization capability of OGA learning, provided the
dictionary is  incoherent and the step-size is appropriated tuned.

\begin{theorem}\label{learning1}
Let $f_{{\bf z},m}^s$ be defined as in Algorithm 1, and $\mathcal
D_N$ be a dictionary with coherence $M$. If $s\leq (2M)^{-1}+1$,
then for all functions $h$ in $\mbox{span}(\mathcal D_N)$, there
holds
$$
               E(\|\Pi_L f^s_{{\bf z},m} -f_\rho\|_\rho^2)\leq
               8\|h-f_\rho\|^2_\rho+\frac{108}{sm}\|h\|_{\mathcal
             L_{1,N}}^2+C\frac{ms\log n}{n},
$$
where $\Pi_L u:=\min\{L,|u|\}sgn(u)$ is the truncation operator at
level $L$,
$$
        \|f\|_{\mathcal
          L_{1,N}}:=\inf\{\sum_{i=1}^N|a_i|:f=\sum_{i=1}^Na_ig_i\},
$$
 and $C$ is a constant depending only on $f_\rho$ and $L$.
\end{theorem}

In Theorem \ref{learning1}, we propose a truncation operator on the
OSGA estimator. It should be noted that such a truncation operator
does not require any computation. Furthermore, as $y \in [-L,L]$,
it is easy to deduce \cite{Zhou2006} that
 $$
    \|\Pi_L f^s_{{\bf z},m}-f_\rho\|^2_\rho\leq \| f^s_{{\bf z},m}-f_\rho\|^2_\rho.
 $$
Theorem \ref{learning1} provides an oracle-type error estimate for
the OSGA learning, since the final error estimate can only be
deduced some oracle about the regression function. We further notice
that up to the constant, the deduced oracle inequality is   the same
as that deduced in \cite{Barron2008} with $k$ in \cite[Theorem
3.1]{Barron2008} replaced by $ms$ in Theorem \ref{learning1}.
Therefore, as a computational burden reduction version of OGA
learning, OSGA learning does not degrade the generalization
capability of OGA learning in the sense that they can obtain the
same learning rate.

To classify the learning rate, we should give some assumptions
(oracle) to the regression function. Along  \cite{Barron2008}'s
flavor, for $r>0$, we define the space $\mathcal L_1^r$ as the set
of all functions $f$ such that, for all $N$, there exists
$h\in\mbox{span}\{D_N\}$ satisfying
\begin{equation}\label{prior}
           \|h\|_{\mathcal L_1}\leq\mathcal B, \ \mbox{and}\
           \|f-h\|_\rho\leq \mathcal BN^{-r}.
\end{equation}
The infimum of all such $\mathcal B$ defines a norm (for $f$ ) on
$\mathcal L_1^r$. Furthermore, let $a\geq 1$ is fixed  we assume
that the size of dictionary, $N$, satisfies $
           N\sim n^a$.
If $f_\rho\in\mathcal L_1^r$,  we can deduce the following learning
rate estimate of OSGA learning.

\begin{corollary}\label{corollary3}
Suppose that the assumptions of Theorem \ref{learning1} holds. If
$f_\rho\in\mathcal L_1^r$, $N\sim n^a$ and $a\geq\frac1{4r}$, then
we can choose $m$ satisfying $m\sim \frac{n^{1/2}}{s}$ such that
$$
               E(\|\Pi_Lf^s_{{\bf z},m} -f_\rho\|_\rho^2)\leq
               C\left(n/\log n\right)^{-1/2}.
$$
where $C$ is a constant depending only on  $L$ and $f_\rho$.
\end{corollary}

It should be highlighted that the main difficulty of OSGA learning
is to select an appropriate iteration number, $m$. Corollary
\ref{corollary3} proposes a strategy of selecting the best $m$, but
the main flaw is that such a choice depends heavily  on the prior
$f_\rho\in \mathcal L_1^r$. In practice, it is usually impossible
to verified. Thus, we turn to pursue a universal strategy to fix
$m$. Hence, we use the same $l^0$ complexity regularization strategy
as that in \cite{Barron2008} to choose $m$.

We define the estimator $\hat{f}=\Pi f_{{\bf z},m^*}$, where $m^*$is
chosen to minimize (over all $m>0$) the penalized empirical risk
\begin{equation}\label{selection k}
       \|y-\Pi_L f^s_{{\bf z},m}\|_n^2+\kappa\frac{ms\log n}{n},
\end{equation}
with $\kappa$ a constant depending only on $L$ and $a$.

Noting Theorem \ref{theorem1}, using the almost same  method as that
in \cite[Theorem 3.1]{Barron2008} we can deduce the following
Theorem \ref{complexity bound}.

\begin{theorem}\label{complexity bound}
If the assumptions in Theorem \ref{learning1} hold, then there
exists a $\kappa $ depending only on $L$ and $a$ such that for all
$m>0$ and $h\in\mbox{span}(\mathcal D_N)$, there holds
$$
               E(\|\hat{f} -f_\rho\|_\rho^2)\leq
               8\|h-f_\rho\|^2_\rho+\frac{108}{sm}\|h\|_{\mathcal
             L_{1,N}}^2+C\frac{ms\log n}{n},
$$
where $C$ is a constant depending only on $\kappa$ and $L$.
\end{theorem}

For the sake of brevity, we omit the proof of Theorem
\ref{complexity bound}. We refer the readers to the proof of Theorem
3.1 in \cite{Barron2008} for the details.   If some assumptions are
added to the regression function $f_\rho$, then we can also deduce
the following learning rate estimate.

\begin{corollary}\label{corollary1}
If $f_\rho\in\mathcal L_1^r$, $a\geq\frac1{4r}$ and the assumptions
of Theorem \ref{complexity bound} holds, there exists a $\kappa $
depending only on $L$ and $a$ such that for all $m>0$ and
$h\in\mbox{span}(\mathcal D_N)$,
$$
               E(\|\hat{f} -f_\rho\|_\rho^2)\leq
               C\left(n/\log n\right)^{-1/2}.
$$
where $C$ is a constant depending only on $\kappa$, $L$ and
$f_\rho$.

\end{corollary}

Corollary \ref{corollary1} together with \cite[Corollary
3.6]{Barron2008} shows that OSGA does not degenerate the learning
performance when compared with   OGA by using the $l^0$ complexity
regularization strategy to fix the iteration number. However, it has
already been pointed out in \cite[Remark 3.5]{Barron2008} that
$\kappa$ should satisfy $\kappa\geq 2568L^4(a+5)$. Such a
pessimistic estimate makes the $l^0$ complexity regularization
strategy (\ref{selection k}) always infeasible. In practice, this
may result in selecting a too small value for $m^*$. Many of the
programmers' spirit will be dampened by this restriction, and shy
away from running OSGA for large $m$. Therefore, the value of the
previous results is only to classify the theoretical feasibility. To
facilitate the use of OSGA, we need to find another strategy to
choose $m$ rather than the $l^0$ complexity regularization
(\ref{selection k}).  A widely used approach is the so-called
``cross-validation'' \cite[Chapter 8]{Gyorfi2002}, which has also
proposed for OGA learning by Barron et al. in \cite{Barron2008}.

%
%Based on Theorem \ref{complexity bound}, using the similar method as
%that in \cite[Theorem 3.8]{Barron2008}, we also deduce that OSGA is
%a universally consistent estimator.
%
%\begin{corollary}\label{Corollary2}
%For an arbitrary regression function, we have
%$$
%         \lim_{n\rightarrow+\infty}  E(\|\hat{f}
%         -f_\rho\|_\rho^2)=0.
%$$
%\end{corollary}

\section{Simulation Supports}

In this section, we present several toy simulations to illustrate
the feasibility,   effectiveness, and   efficiency of   OSGA
learning. The main purpose can be divided into three aspects. The
first one is to reveal that there exists a relationship between the
``dictionary-selection'' stage and  ``greedy-definition'' stages for
greedy learning. Since the incoherence assumption is too strict to
describe the property of the dictionary and difficult to verify,
especially for supervised learning \cite{Mougeot2012}, we do not
implement the simulation for dictionaries with such a pessimistic
  assumption. Instead, we utilize two widely used
dictionaries such as the trigonometric polynomial dictionary and
Gaussian radial basis function dictionary to justify our viewpoint.
The second one is to analyze the pros and cons of OSGA learning. In
particular, we compare both the training time and test time between
OSGA learning and OGA learning with different dictionaries for
different regression functions. The last one is to compare the
performance of OSGA with other typical dictionary learning strategy
such as  the OGA learning \cite{Barron2008}, Lasso
\cite{Tibshirani1995}, ridge regression \cite{Scholkopf2001}, bridge
regression \cite{Daubechies2010} (for example, the half coefficient
regularization \cite{Xu2012}) and greedy boosting
\cite{Friedman2001}. It should be noted that the aim of this
simulation is to compare different learning strategies for fixed
dictionaries, thus we only employ two fixed dictionaries rather than
pursuing the best dictionary. Given the dictionary, we then analyze
  different performances of the aforementioned different learning
schemes.

\subsection{Experiment Setup}
\emph{Data sets:} we employ two regression functions $f_{\rho }$ as
 \begin{equation*}
               f_1(x)={\rm sinc}(40x-10)+{\rm sinc}(60x-30)+{\rm sinc}(20x-1)+\cos(10x),
 \end{equation*}
where ${\rm sinc} (t)=\frac{\sin t}{t}$, and
 \begin{equation*}
        f_2(x)=\left\{\begin{array}{cc}
        1/3-x, & \mbox{if}\ 0\leq x<1/3,\\
         x^2, & \mbox{if}\  1/3\leq x \leq 2/3;\\
          -1, & 2/3<x\leq 1.\end{array}\right.
 \end{equation*}
It is easy to see that $f_1$ is an infinitely differential function
and $f_2$ is a discontinuous function. We generated the training
sample set ${\bf z}=\{(x_{i},y_{i})\}_{i=1}^{5000}$ through
independently and randomly sampling $x_{i}$ from $\mathcal{U}(0,1)$
, and the corresponding $y_{i}^{\prime }s$ to be $y_{i}=f_{\rho
}(x_{i})+\epsilon $, with $\varepsilon \sim N(0,0.1)$ being the
white noise. The learning performances of different algorithms
were then tested by applying the resultant estimators to the test set ${\bf z}_{test}=%
\{(x_{i}^{(t)},y_{i}^{(t)})\}_{i=1}^{5000}$ which was generated
similarly to ${\bf z}$ but with a promise that $y_{i}^{\prime }s$ were always taken to be $%
y_{i}^{(t)}=f_{\rho }(x_{i}^{(t)}).$

\emph{Dictionary:} In each simulation, we adopt two types of
dictionaries. The first one is constructed by the Gaussian radial
basis function: $\exp\{-\sigma|x-t_i|^2\}$ with
$\{t_i\}_{i=1}^{500}$ being drawn identically and independently from
$\mathcal{U}(0,1)$ and appropriately selected $\sigma$. Indeed, we
set $\sigma=200$ for $f_1$ and $\sigma=1000$ for $f_2.$ The other
one is constructed by the trigonometric polynomial: $\cos kt$ with
$k\in\{1,2,...,500\}$.

\emph{Methods:} For OSGA and OGA learning, we applied $QR$
decomposition to acquired the least squared estimates
\cite{Sauer2006}. For greedy boosting, we used the L2boost algorithm
\cite{Friedman2001} with the step size $0.0005$. For $L_2$
coefficient regularization (or ridge regression), we use its
analytic regularized least square solution \cite{Cucker2001}. For
$L_1$ (or lasso) and $L_{1/2}$ coefficient regularization schemes,
we utilize the iterative soft \cite{Daubechies2004} and half
\cite{Xu2012} thresholding algorithms to obtain the corresponding
estimators, respectively.

\emph{Measurements:} Since the aim of the toy simulations is to
justify the feasibility of OSGA, we don't access any concrete
parameter-selection strategies. Therefore, we draw our conclusion in
the basis of the most appropriate parameters. Under this
circumstance, we do not divide the training set into training data
and validation data and use validation data to choose parameters as
in \cite{Gyorfi2002}, instead, we use the test set to fix parameters
directly. To be detailed, the iteration steps for OSGA, OGA and
greedy boosting, and the regularization parameters $\lambda \in
\{2^{-10},2^{-9},...,2^{10}\}$ for $L_q$ coefficients regularization
with $q=1/2,1,2$ are selected when the prediction error on the test
data is minimized.
%In all the following simulations, we find the "best" model
%parameters (e.g., , and the regularization parameters $\lambda \in
%\{2^{-10},2^{-9},...,2^{10}\}$ for $L_q$ coefficients
%regularization) that minimizes the prediction error on the test
%sets. \textbf{This is because we are more interested in the
%relationship between dictionary-selection and greedy-definition in
%this paper. However, the existing stopping rules for OGA-type
%learning are not stable enough to support our analysis. Therefore,
%we compare their performances in ideal conditions.}
We recorded the rooted mean squared error (RMSE) of test error, the
sparsity of the coefficients, and the training time under the selected parameters to measure the
performances of the mentioned learning schemes.

\emph{Environment:} All the simulations and experiments   were
conducted in Matlab R2013a on a desktop computer with Windows
7/Intel(R)/Core(TM) i7-3770K RAM and 3.50GHz CPU, and the statistics
were averaged based on 100 independent trials.

\subsection{The relationship between dictionary-selection and greedy-definition}

Theorem \ref{learning1} theoretically presents that if the
relationship between the ``dictionary-selection'' and
``greedy-definition'' stages is considered, then the efficiency of
greedy learning can be essentially improved. However, such a
theoretical result is built on the incoherence property of the
dictionary. As is shown in \cite{Mougeot2012}, the incoherence
assumption  in the background of supervised learning is too strict
to describe the property of dictionaries. We guess that there may
exist a much looser measurement than it within our purpose. To
verify this conjecture, we employ both trigonometric polynomial
dictionary (TPD) and Gaussian radial basis dictionary (GRD) to be
the carriers of OSGA. It can be found in \cite{Bass2005} that the
TPD dictionary together with the random samples can develop  a
well-conditioned sampling matrix \cite[Theorem 5.1]{Bass2005}, while
the sampling matrix constructed by GRD is usually ill-conditioned
\cite{Minh2010}.  We compare  the TPD and GRD for OSGA learning, so
as to experimentally study how the ``dictionary-selection'' stage
influences the ``greedy-definition'' stage. Fig.\ref{fig_sinc} and
Fig.\ref{fig_piece} summarize the learning rate of OSGA for the
continuous regression function $f_1$ and the discontinuous
regression function $f_2$, respectively.
\begin{figure}[!t]
\begin{minipage}[b]{.45\linewidth}
\includegraphics*[scale=0.45]{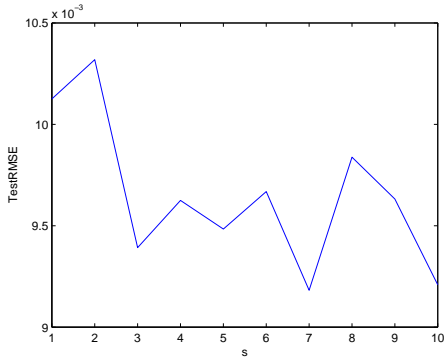}
%  \vspace{-.5cm}
\centerline{{\small (a) RMSE as a function of step-size $s$ (TPD)   }}
\end{minipage}
\hfill
\begin{minipage}[b]{.45\linewidth}
\includegraphics*[scale=0.45]{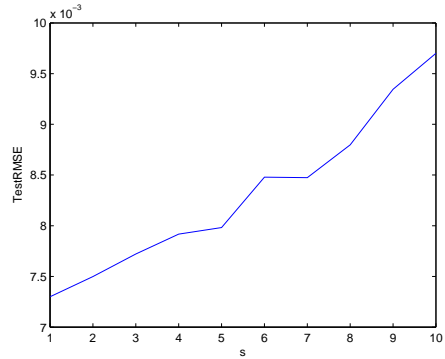}
%  \vspace{-.5cm}
\centerline{{\small (b) RMSE as a function of step-size $s$ (GRD)  }}
\end{minipage}
\hfill \caption{ Learning performance of OSGA for continuous
regression function $f_1$}
\label{fig_sinc}
\end{figure}

\begin{figure}[!t]
\begin{minipage}[b]{.45\linewidth}
\includegraphics*[scale=0.45]{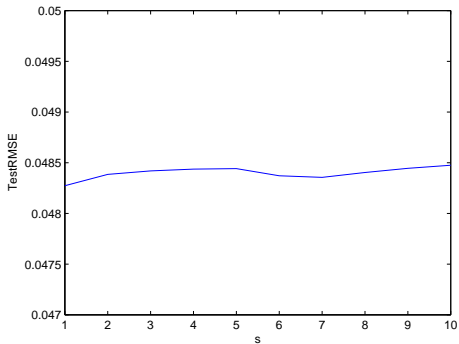}
%  \vspace{-.5cm}
\centerline{{\small (a) RMSE as a function of step-size $s$ (TPD)   }}
\end{minipage}
\hfill
\begin{minipage}[b]{.45\linewidth}
\includegraphics*[scale=0.45]{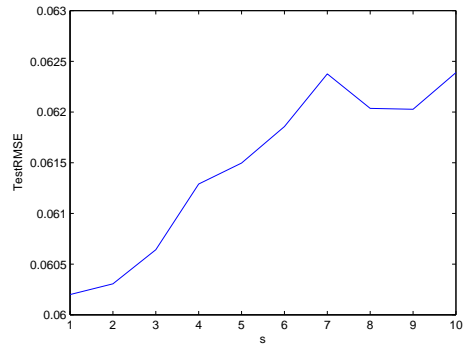}
%  \vspace{-.5cm}
\centerline{{\small (b) RMSE as a function of step-size $s$ (GRD)  }}
\end{minipage}
\hfill \caption{ Learning performance of OSGA for discontinuous
regression function $f_2$}
\label{fig_piece}
\end{figure}

 It is seen from Fig.\ref{fig_sinc}(a) that when the TPD is selected, the RMSE does not
 increase as the step-size $s$ increases. However, Fig.\ref{fig_sinc}(b) shows
that when the GRD is selected, the RMSE increases obviously. A
similar trend can also be observed when applying the OSGA on the
discontinuous regression function $f_2$, as shown in
Fig.\ref{fig_piece}. From an experimental viewpoint, such
differences demonstrate that a dictionary with certain good
properties  do help to redefine what is greedy in the
``greedy-definition'' stage and improve the efficiency of greedy
learning  while a ``bad'' dictionary can not bring such benefits.
All these simulations reveal that there does exist a relationship
between the ``dictionary-selection'' and``greedy-definition''
stages.

It should be noted from Fig.\ref{fig_sinc} (b) and
Fig.\ref{fig_piece} (b) that even for the GRD, OSGA does not degrade
the generalization capability of OGA very much. Indeed, taking
Fig.\ref{fig_sinc} for example, from $s=1$ to $s=10$, the RMSE only
increases from 0.0073 to 0.0097. This phenomenon shows that at the
 cost of a small loss of generalization capability, OSGA   provides a
possibility to reduce the computation burden of OGA learning, even
for some ``bad'' dictionaries. The main reason of this phenomenon,
from our point of view, can be stated as follows. As   $s$
increases, due to the high-correlation of the atoms of GRD, the
selected atoms via (\ref{select s atom}) are high-correlated to the
residual $r_{{\bf z},k-1}^s,$ so they are high-correlated with each
other. Therefore, the approximation capability of the spanned space
$V_{sk}$ are not much better than $V_{s(k-1)+1}$, where
$V_{s(k-1)+1}$ denotes the span of dictionary in which only one atom
is added to $V_{s(k-1)}=\{g_1,\dots,g_{s(k-1)}\}$ according to the
classical greedy definition. However, as the atoms of $V_{sk}$ is
more than that of $V_{s(k-1)+1}$, and the high-correlation of GRD,
the capacity of $V_{sk}$  is only larger than that of $V_{s(k-1)+1}$
to a limited extent. Thus, according to the known bias and variance
trade-off principle \cite{Cucker2007}, the bias decreases a little
while the variance increases a little, which makes the final
generalization error varies only a little. As a consequence, more
atoms are required to reach a good prediction as compared with OGA.

\subsection{The pros and cons  of OSGA Learning}

The main motivation to introduce OSGA to tackle supervised learning
problem is that OSGA can reduce the computational burden of OGA,
provided the dictionary possess some prominent property. The main
purpose of this series of simulations is to verify this pros of OSGA
learning. Furthermore, we also experimentally analyze the cons of
OSGA learning. To this end, we aim to compare both the training time
and test time of OSGA learning with different step-size, $s$. As the
test time only depends on the sparsity of the coefficients of the
deduced estimator, we record both the training time
(Fig.\ref{fig_time}) and sparsity of the estimator
(\ref{fig_sparsity}) as a function of $s$.

It can be found in Fig.\ref{fig_time} that,  to deduce the OSGA
estimator, the training time monotonously decreases with respect to
$s$. This implies that as far as the training time is concerned,
OSGA learning outperforms than OGA learning. The reason is that
OSGA learning   can skip many least square estimation for a large
$s$, as compared with the standard OGA.  This conclusion is regarded
as the main pros of OSGA learning. On the other hand, as shown in
Fig.\ref{fig_sparsity}, the number of the selected atoms of the OSGA
estimator may be a bit larger than OGA, which is witnessed in
Fig.\ref{fig_sparsity} (a), (b), (d). Under this circumstance, the
test cost of OSGA learning is larger than that of OGA learning,
which can be considered as the main cons of OSGA learning.

\begin{figure}[!t]
\begin{minipage}[b]{.45\linewidth}
\includegraphics*[scale=0.45]{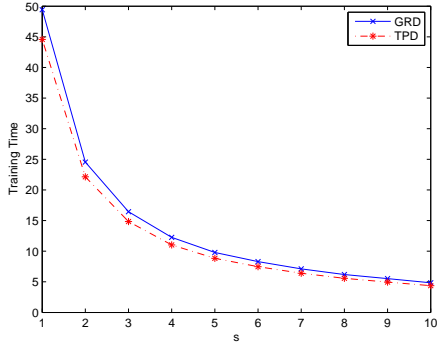}
%  \vspace{-.5cm}
\centerline{{\small (a) Training time for $f_1$   }}
\end{minipage}
\hfill
\begin{minipage}[b]{.45\linewidth}
\includegraphics*[scale=0.45]{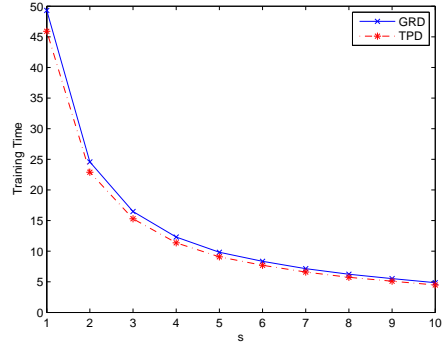}
%  \vspace{-.5cm}
\centerline{{\small (b) Training time for $f_2$   }}
\end{minipage}
\hfill \hfill \caption{Comparison of training time }
\label{fig_time}
\end{figure}

\begin{figure}[!t]
\begin{minipage}[b]{.45\linewidth}
\includegraphics*[scale=0.45]{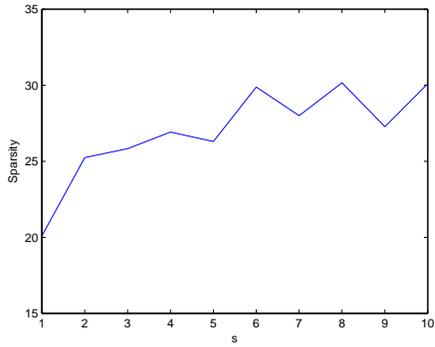}
%  \vspace{-.5cm}
\centerline{{\small (a) TPD  for $f_1$     }}
\end{minipage}
\hfill
\begin{minipage}[b]{.45\linewidth}
\includegraphics*[scale=0.45]{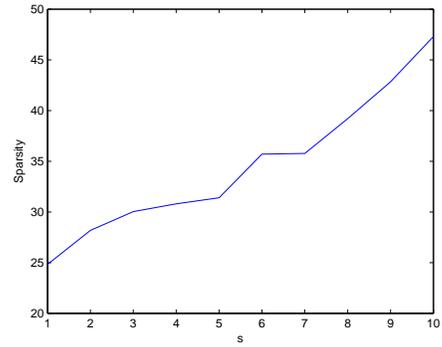}
%  \vspace{-.5cm}
\centerline{{\small (b)   GRD for $f_1$ }}
\end{minipage}
\hfill
\begin{minipage}[b]{.45\linewidth}
\includegraphics*[scale=0.45]{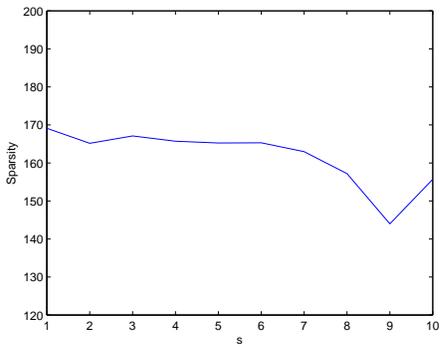}
%  \vspace{-.5cm}
\centerline{{\small (c)   TPD for $f_2$   }}
\end{minipage}
\hfill
\begin{minipage}[b]{.45\linewidth}
\includegraphics*[scale=0.45]{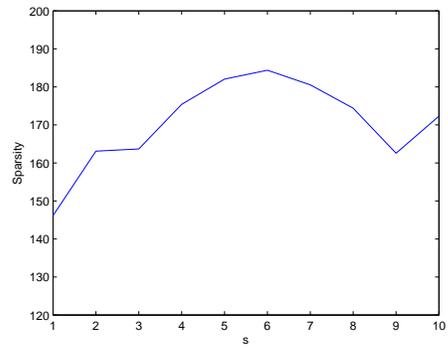}
%  \vspace{-.5cm}
\centerline{{\small (d)   GRD for $f_2$ }}
\end{minipage}
\hfill \caption{ The coefficient sparsity as a function of $s$  }
\label{fig_sparsity}
\end{figure}

The above two simulations only take the computational burden for
OSGA into account. We further do the following simulations to
consider both the computational burden and generalization
capability. Since the sparsity of the OSGA estimator is the product
of the iteration number $m$ and step-size $s$,   the training time
also depends heavily on the sparsity. Therefore, we employ a
simulation by setting RMSE as a function of the sparsity, $k$. Such
a simulation, shown in Fig.\ref{fig_detail} presents a summary of
OSGA learning. It is seen that the number of atoms required for the
smallest RMSE, may grow  as the step-size $s$ increases (see (a),
(b), (d) in Fig.\ref{fig_detail}). Furthermore,  the obtained least
RMSE for different $s$ varies very little. These show  that OSGA
learning can reduce the computational burden of OSGA without
sacrificing the generalization capability very much, and the   price
(or risk) to do such a reduction is that the test time may increase.
It should be highlighted in Fig.\ref{fig_detail} that, for the well
developed dictionary, TPD, and a suitable regression function, all
the training time, test time and RMSE can be reduced by utilizing
OSGA. The reason for these phenomenons is similar as that  presented
in the bottom of the last subsection. For ``bad'' dictionary such as
GRD,
 the bias decreases a little while the variance
increases a little as $s$ increases. Due to the high-correlation of
GRD, the quantity of decreased bias is smaller than that of
increased variance. Thus, it requires more atoms and makes the
generalization error a little larger. For a ``good'' dictionary such
as TPD, and an appropriate regression function, the quantity of
decreased bias can be larger than that of increased variance, which
leads to both smaller sparsity of the estimator and less
generalization error.

\begin{figure}[!t]
\begin{minipage}[b]{.45\linewidth}
\includegraphics*[scale=0.45]{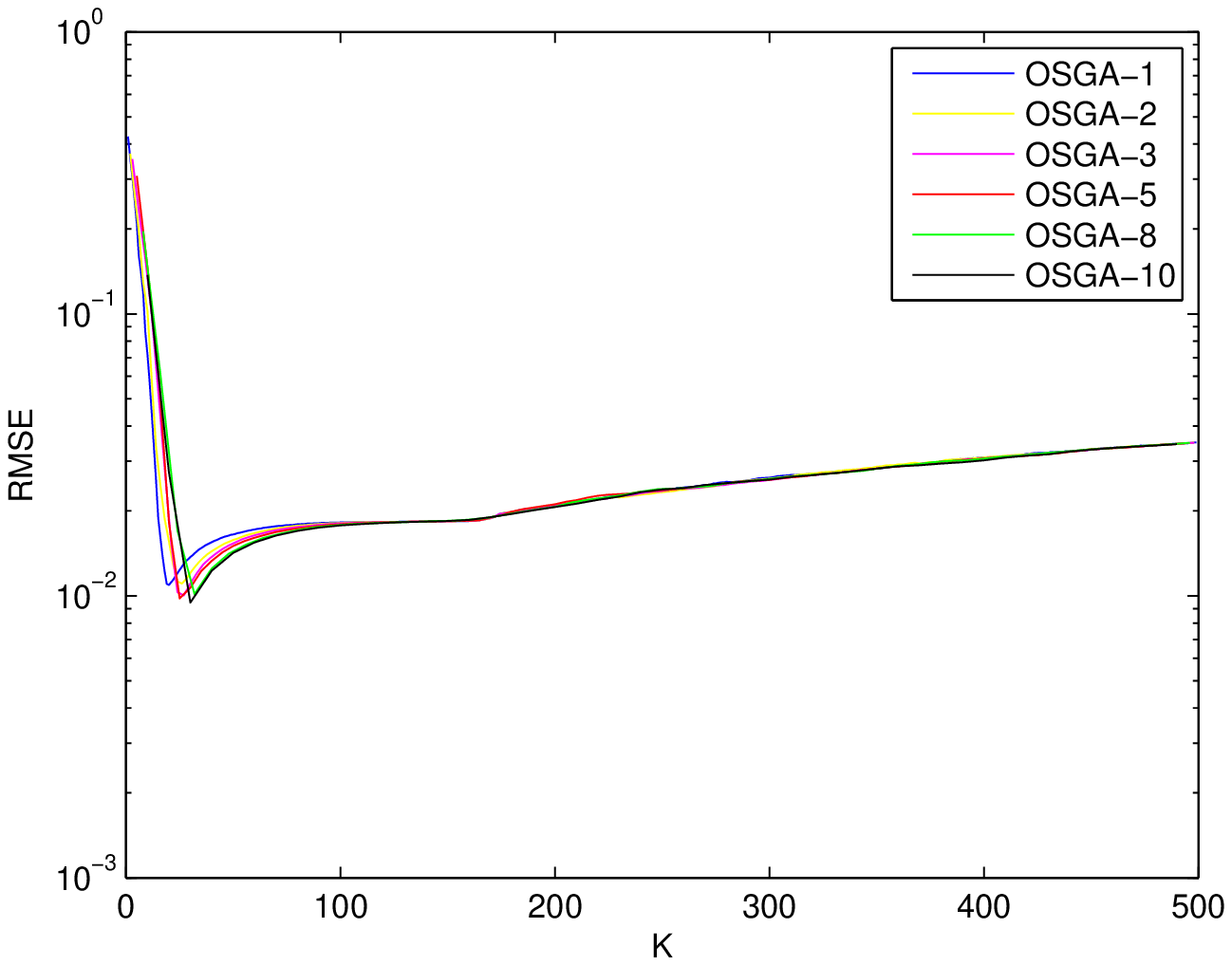}
%  \vspace{-.5cm}
\centerline{{\small (a)  TPD for $f_1$   }}
\end{minipage}
\hfill
\begin{minipage}[b]{.45\linewidth}
\includegraphics*[scale=0.45]{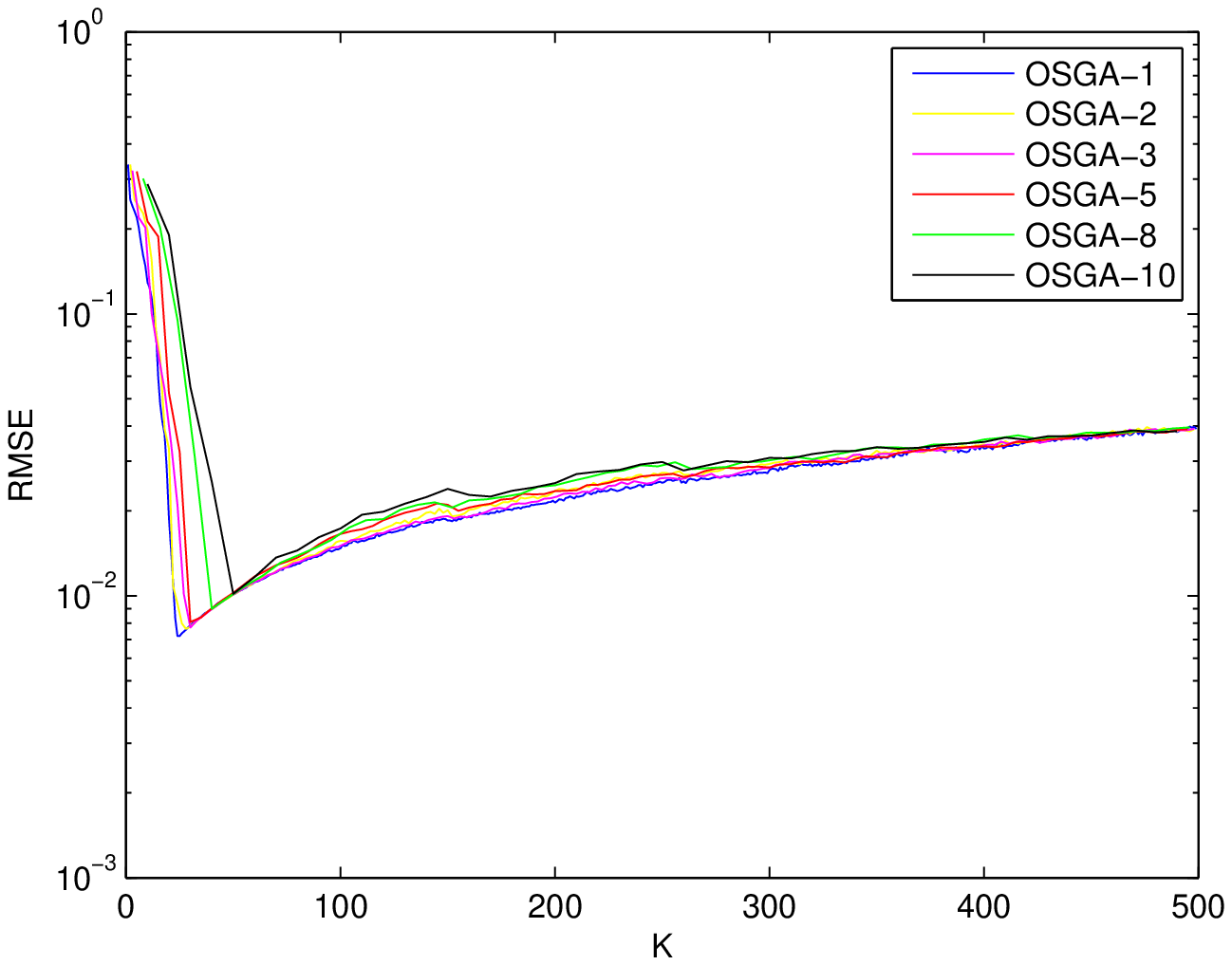}
%  \vspace{-.5cm}
\centerline{{\small (b)    GRD for $f_1$  }}
\end{minipage}
\hfill
\begin{minipage}[b]{.45\linewidth}
\includegraphics*[scale=0.45]{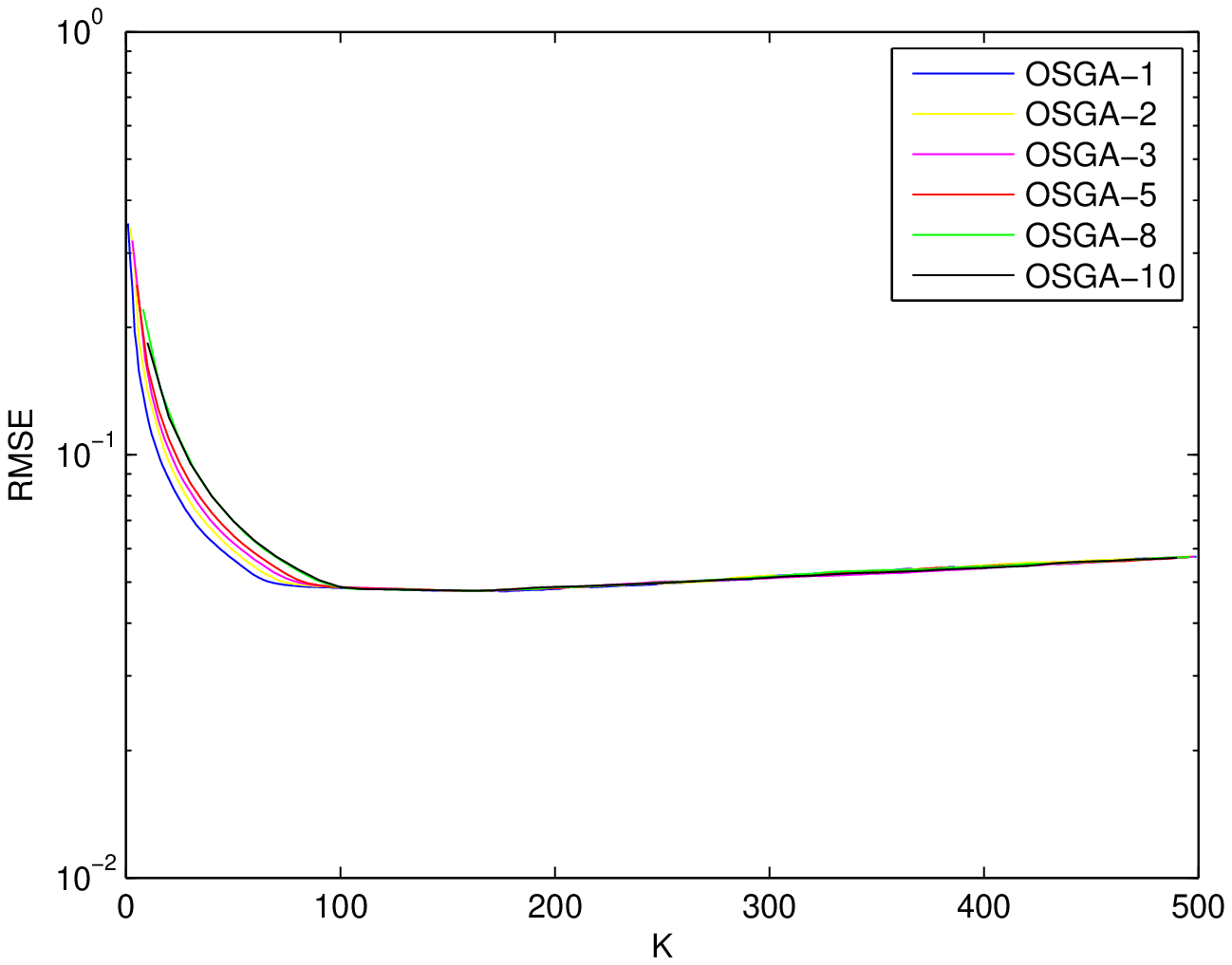}
%  \vspace{-.5cm}
\centerline{{\small (c)   TPD for $f_2$ }}
\end{minipage}
\hfill
\begin{minipage}[b]{.45\linewidth}
\includegraphics*[scale=0.45]{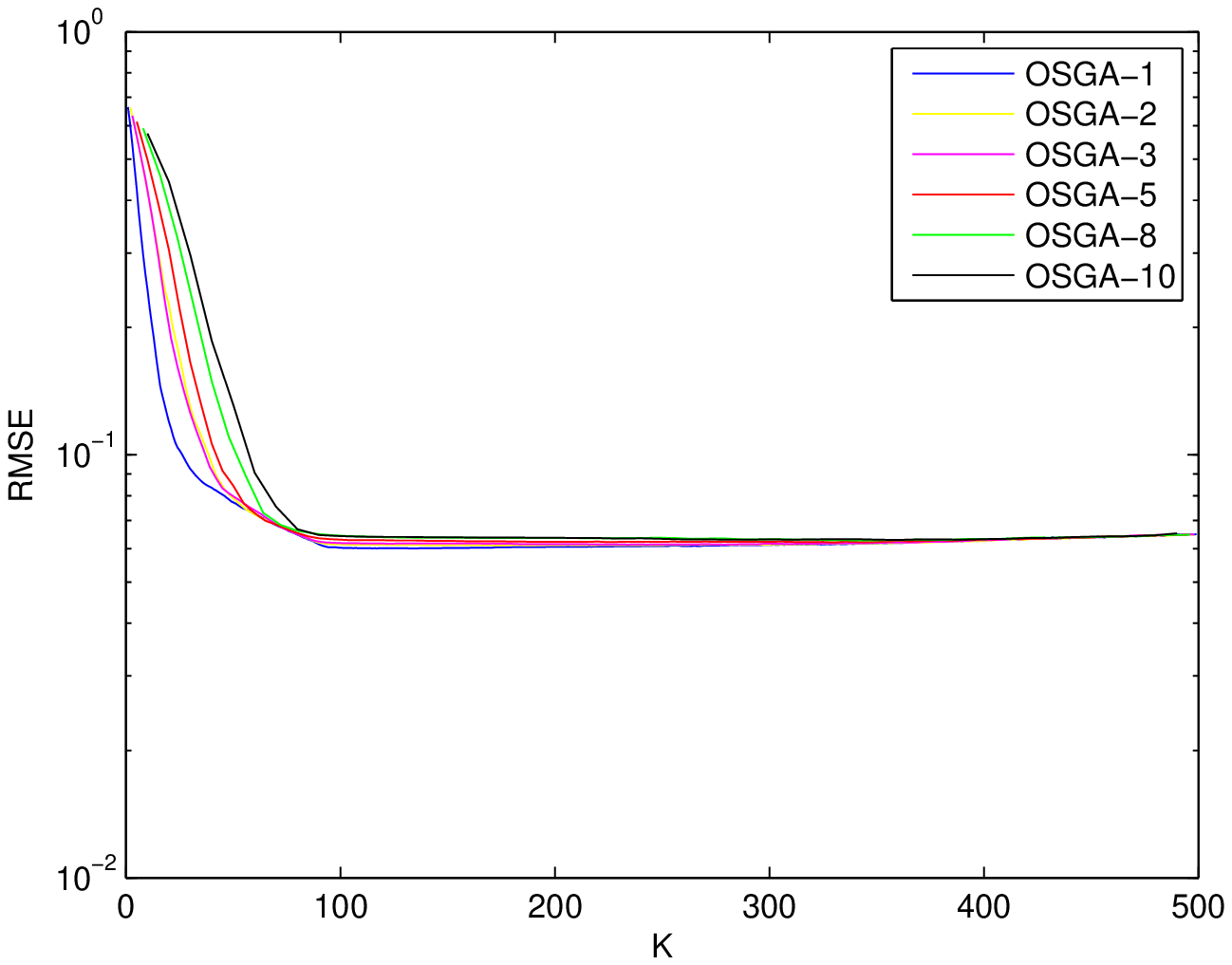}
%  \vspace{-.5cm}
\centerline{{\small (d) GRD for $f_2$ }}
\end{minipage}
\hfill \caption{  RMSE as a function of the sparsity}
\label{fig_detail}
\end{figure}

\subsection{The generalization ability of OSGA Learning}

Finally,  we assess the generalization capacity of OSGA learning as
compared with some typical dictionary-based learning methods. As the
purpose of this paper is not to pursue the best dictionary, we just
employ two fixed dictionaries as GRD and TPD. Specifically, we run
the OSGA-1 (or OGA), OSGA-2, OSGA-5, OSGA-10,  Lasso, ridge
regression, half regression, and greedy boosting on the same data
and dictionaries. Here, OSGA-s denotes that there are $s$ atoms
selected in the ``greedy-definition'' stage of OSGA. The results are
summarized in Tables 1 and 2. It can be found in Tables 1 and 2 that
for GRD, the test error of OSGA-s increases as $s$ increases, while
for TPD, the test error of OSGA-s monotonously decreases with
respect to $s$, which verifies our assertion proposed in Section 4.2
further. Moreover, it is shown in Tables 1 and 2 that the
performance of OSGA  is similar to  other competitive methods. This
shows that OSGA can reduce the computational burden without
sacrificing the generalization capability very much.

It should be noted that the results of OSGA is searched in the whole
finite discrete parameter space, and therefore it is easy for OSGA
to select the best parameter. However,  for both greedy boosting and
$L_q$ coefficient regularization, their main parameters are
distributed in   continuous (or infinite) spaces, which makes the
best parameters be difficult to achieve. This phenomenon can be
regarded as another advantage of OSGA, and also give a  reason why
$L_q$ coefficient regularization and greedy boosting's
generalization capability seems worse than OSGA in the second column
of Table 2. We believe that if a more elaborate parametric selection
method is given, then the RMSE of these methods can reduce.

\begin{table}[h]
\renewcommand{\arraystretch}{1}
\caption{Simulation results on the regression function $f_1$}
\label{table:result1}
\centering
\begin{tabular}{|c||c||c||c||c|}
\hline
Methods  & RMSE (GRD) & Sparsity (GRD) & RMSE (TPD) & Sparsity (TPD)\\
\hline
OSGA-1 &0.0073     &24.83   &0.0101  &20.09 \\
\hline
OSGA-2 &0.0075     &28.18   &0.0103  &25.24 \\
\hline
OSGA-5 &0.0080  &31.40  &   0.0095  &26.30  \\
\hline
OSGA-10 &0.0097     &43.70  &   0.0092  &30.10  \\
\hline
$L_2$ regularization &0.0069     &500.00    &0.0181  &500.00 \\
\hline
$L_1$ regularization &0.0075& 500.00& 0.0125& 69.03\\
\hline
$L_{1/2}$ regularization &0.0075     &475.16   &0.0118  &25.68\\
\hline
Greedy boosting &0.0080  &442.63   &0.0124  &78.16\\
\hline
\end{tabular}
\end{table}

\begin{table}[h]
\renewcommand{\arraystretch}{1}
\caption{Simulation results on the regression function $f_2$}
\label{table:result2}
\centering
\begin{tabular}{|c||c||c||c||c||c||c||c||c|}
\hline
Methods & RMSE (GRD) & Sparsity (GRD) & RMSE (TPD) & Sparsity (TPD)\\
\hline
OSGA-1  &0.0597   &146.16   &0.0483     &169.11\\
\hline
OSGA-2  &0.0603   &163.10   &0.0484     &165.16\\
\hline
OSGA-5   &0.0615  &  182.05 &0.0484 & 165.25\\
\hline
OSGA-10  &0.0624  &  172.30 &0.0485 & 155.70\\
\hline
$L_2$ regularization  &0.0763   &500.00   &0.0485     &500.00\\
\hline
$L_1$ regularization & 0.0824& 500.00& 0.0484& 268.20\\
\hline
$L_{1/2}$ regularization  &0.0823   &482.97   &0.0485     &170.71\\
\hline
Greedy boosting &0.0848   &450.78   &0.0485     &320.12\\
\hline
\end{tabular}
\end{table}

%\begin{table}[h]
%% increase table row spacing, adjust to taste
%\renewcommand{\arraystretch}{1}
%\caption{Primary Parameters of SAR system and geometry}
%\label{table:para}
%\centering
%\begin{tabular}{|c||c||c||c||c||c||c||c||c|}
%\hline
%Methods & RMSE (GRD) & Sparsity (GRD) & RMSE (TPD) & Sparsity (TPD) & RMSE (GRD) & Sparsity (GRD) & RMSE (TPD) & Sparsity (TPD)\\
%\hline
%OSGA-1 &0.0073     &24.8300   &0.0101  &20.0900 &0.0597   &146.1600   &0.0483     &169.1100\\
%\hline
%OSGA-2 &0.0075     &28.1800   &0.0103  &25.2400 &0.0603   &163.1000   &0.0484     &165.1600\\
%\hline
%OSGA-5 &0.0080     &31.4000    &   0.0095  &26.3000  &0.0615  &  182.0500 &0.0484 & 165.2500\\
%\hline
%OSGA-10 &0.0097    &43.7   &   0.0092  &30.1  &0.0624  &  172.3 &0.0485 & 155.7\\
%\hline
%$L_2$ regularization &0.0069     &500.0    &0.0181  &500.00 &0.0763   &500.00   &0.0485     &500.00\\
%\hline
%$L_1$ regularization &0.0075& 500.0& 0.0125& 69.0300& 0.0824& 500.00& 0.0484& 268.2000\\
%\hline
%$L_{1/2}$ regularization &0.0075     &475.1600   &0.0118  &25.6800 &0.0823   &482.9700   &0.0485     &170.7100\\
%\hline
%Greedy boosting &0.0080     &442.6300   &0.0124  &78.16 &0.0848   &450.7800   &0.0485     &320.12\\
%\hline
%\end{tabular}
%\end{table}

\section{Proofs}

To prove Theorem \ref{theorem1}, we need the following three lemmas.
The first and third lemmas can be found in \cite{Donoho2007} and
\cite{Devore1996}, respectively.

\begin{lemma}\label{coherence and RIP}
Assume a dictionary $\mathcal D$ has coherence $M$. Then we have for
any distinct $g_j\in\mathcal D, j=1,2,\dots, s$ and for any $a_j$,
$j=1,2\dots,s$, the inequalities
$$
      (1-M(s-1))\sum_{i=1}^sa_i^2\leq\left\|\sum_{i=1}^sa_ig_i\right\|^2\leq
      (1+M(s-1))\sum_{i=1}^sa_i^2.
$$
\end{lemma}

\begin{lemma}\label{important estimation}
Assume a dictionary $\mathcal D$ has coherence $M$. Let
$\{g_i\}_{i=1}^s\subset\mathcal D$ and
$G(s):=\mbox{span}\{g_1,\dots,g_s\}$. Then we have
$$
       \frac1{1+M(s-1)}\sum_{i=1}^s\langle
       f,g_i\rangle^2\leq\|P_{G(s)}(f)\|^2\leq\frac1{1-M(s-1)}\sum_{i=1}^s\langle
       f, g_i\rangle^2.
$$
\end{lemma}

\begin{proof} It follows from the definition of $P_{G(s)}(f)$
that
$$
          \|P_{G(s)}(f)\|=\max_{\psi\in G(s),\|\psi\|\leq 1}|\langle
          f,\psi\rangle|.
$$
Let $\psi=\sum_{i=1}^sa_ig_i$. Then it follows form Lemma
\ref{coherence and RIP} that for arbitrary $\psi\in G(s)$ and
$\|\psi\|\leq 1$, there holds
$$
       |\langle f,\psi\rangle|^2=\left|\left\langle
       f,\sum_{i=1}^sa_ig_i\right\rangle\right|^2
       =\left|\sum_{i=1}^sa_i\left\langle
       f,g_i\right\rangle\right|^2
       \leq\sum_{i=1}^sa_i^2\sum_{i=1}^s\langle f,g_i\rangle^2
       \leq
        \frac1{1-M(s-1)}\sum_{i=1}^s\langle f,g_i\rangle^2.
$$
Therefore, we get
$$
          \|P_{G(s)}(f)\|^2\leq\frac1{1-M(s-1)}\sum_{i=1}^s\langle
          f,g_i\rangle^2.
$$
To bound $\|P_{G(s)}(f)\|^2$ from below, noting
$$
           \left\|\sum_{i=1}^s\langle f, g_i\rangle
           g_i\right\|^2\leq (1+M(s-1))\sum_{i=1}^s\langle
           f,g_i\rangle^2,
$$
we have
$$
          \left\|\sum_{i=1}^s\frac{\langle f,g_i\rangle}{(1+M(s-1))^{1/2}\left|\sum_{i=1}^s\langle
          f,g_i\rangle^2\right|^{1/2}}g_i\right\|\leq 1.
$$
Thus, there holds
\begin{eqnarray*}
         \|P_{G(s)}(f)\|^2
         &=&
         \max_{\psi\in G(s),\|\psi\|\leq 1}|\langle
          f,\psi\rangle|
          \geq
          \left|\left\langle f,\sum_{i=1}^s\frac{\langle f,g_i\rangle}{(1+M(s-1))^{1/2}\left|\sum_{i=1}^s\langle
          f,g_i\rangle^2\right|^{1/2}}g_i\right\rangle\right|^2\\
          &=&
          \left|\sum_{i=1}^s\frac{\langle
          f,g_i\rangle^2}{(1+M(s-1))^{1/2}\left|\sum_{i=1}^s\langle
          f,g_i\rangle^2\right|^{1/2}}\right|^2\\
          &=&
          \left|(1+M(s-1))^{-1/2}\left(\sum_{i=1}^s\langle
          f,g_i\rangle^2\right)^{1/2}\right|^2\\
          &=&
          (1+M(s-1))^{-1}\sum_{i=1}^s\langle f,g_i\rangle^2.
\end{eqnarray*}
This finishes the proof of Lemma \ref{important estimation}
\end{proof}

\begin{lemma}\label{induction}
Let $(a_n)_{n\neq0}$ be a set of decreasing nonnegative numbers that
satisfy $a_0\leq R$ and $a_k\leq
a_{k-1}\left(1-\frac{a_{k-1}}{R}\right)$ for all $k>0$. Then, for
all $n>0$, there holds $a_n\leq \frac{R}{n+1}$.
\end{lemma}

By the help of the above lemmas, we are in a position to give the
proof of Theorem \ref{theorem1}.

\begin{proof}[Proof of Theorem \ref{theorem1}]
Let $P_f$ be the projection of $f$ onto $\mathcal L_1$. Noting that
every element of $\mathcal L_1$ can be approximated arbitrarily well
by elements of the form
$$
         \phi=\sum_{j=1}^\infty c_jg_j, g_j\in\mathcal
         D,\sum_{j=1}^\infty|c_j|\leq\|P_f\|_{\mathcal L_1}+\delta,
         |c_1|\geq|c_2|\geq\dots,
$$
where $\delta>0$ is arbitrary positive number. It will be clear from
the following argument that it is sufficient to consider elements
$P_f$ of the above form. Suppose $v$ is such that
$$
            |c_v|\geq\frac{2(\|P_f\|_{\mathcal
            L_1}+\delta)}{s}\geq|c_{v+1}|.
$$
Then, the aforementioned assumption on the sequence $\{c_j\}$ yields
that $v\leq\frac{s}2$ and $|c_{s+1}|<\frac{\|P_f\|_{\mathcal
L_1}+\delta}{s}$. We claim that elements $g_1,\dots,g_v$ will be
chosen among $\varphi_1,\dots,\varphi_s$ at the first iteration.
Indeed, for $j\in[1,v]$, we have
\begin{eqnarray*}
      |\langle P_f,g_j\rangle|
      &=&
      \left|\left\langle\sum_{k=1}^\infty
      c_kg_k,g_j\right\rangle\right|=\left|c_j+\sum_{k\neq j}
      c_k\langle g_k,g_j\rangle\right|\\
      &\geq&
      |c_j|-\left|\sum_{k\neq j}c_k\langle g_k,g_j\rangle\right|
      \geq
      |c_j|-M(\|P_f\|_{\mathcal L_1}+\delta-|c_j|)\\
      &=&
      |c_j|(1+M)-M((\|P_f\|_{\mathcal L_1}+\delta)
      \geq
      \frac{2(\|P_f\|_{\mathcal L_1}+\delta)}{s}(1+M)-M(\|P_f\|_{\mathcal
      L_1}+\delta).
\end{eqnarray*}
For all $g$ distinct from $g_1,\dots,g_s$, we have
$$
         |\langle P_f,g\rangle|=\left|\sum_{k=1}^\infty c_k\langle
         g_k,g\rangle\right|<\frac{(\|P_f\|_{\mathcal
         L_1}+\delta)}{s}+M(\|P_f\|_{\mathcal L_1}+\delta)=
         (\|P_f\|_{\mathcal L_1}+\delta)(M+\frac1s).
$$
Since $s\leq\frac1{2M}+1$, we obtain
$$
        |\langle P_f,g_i\rangle|\geq|\langle P_f,g\rangle|
$$
for all $g$ distinct from $g_1,\dots,g_s$. This implies that
$$
        |\langle f,g_i\rangle|\geq|\langle f,g\rangle|
$$
for all $g$ distinct from $g_1,\dots,g_s$.
 Thus, we do not pick any
$g\in\mathcal D$ distinct from $g_1,\dots,g_s$ until we have chosen
all $g_1,\dots,g_v$.

Now we proceed the proof of Theorem \ref{theorem1}. Denote
$F_m=\mbox{span}(\varphi_i,i\in I_m)$. Then, $H_{m-1}, F_m\in H_m$.
Therefore,
$$
         r_m^s=f-P_{H_{m}}(f)=r_{m-1}^s+f_{m-1}^s-P_{H_m}(r_{m-1}^s+f_{m-1}^s)
         =r_{m-1}^s-P_{H_m}(r_{m-1}^s).
$$
It is clear that the inclusion $F_m\subset H_m$ implies
$$
             \|r_m^s\|\leq\|r_{m-1}^s-P_{F_m}(r_{m-1}^s)\|.
$$
Using the notation $p_m=P_{F_m}(r^s_{m-1})$. We continue
$$
       \|r_{m-1}^s\|^2=\|r_{m-1}^s-p_m\|^2+\|p_m\|^2
$$
and
$$
         \|r_m^s\|^2\leq\|r_{m-1}^s\|^2-\|p_m\|^2.
$$

It is obvious that for arbitrary $h\in\mathcal L_1$,
$$
        \|r_{m-1}^s\|^2=\langle r_{m-1}^s,f\rangle=\langle
        r_{m-1}^s,h+f-h\rangle =\langle r_{m-1}^s,h\rangle+\langle
        r_{m-1}^s, f-h\rangle.
$$
The known Cauchy-Schwarz inequality implies that
\begin{equation}\label{first term}
               \langle
               r_{m-1}^s,f-h\rangle\leq\|r_{m-1}^s\|\cdot\|f-h\|.
\end{equation}

Now we turn to bound $\langle r_{m-1}^s,h\rangle$. Denote
$J_l=[(l-1)s+v+1,ls+v]$, $G(J_l)=\mbox{span}\{g_i\}_{i\in J_l}$, and
$$
         q_s=q_s(r_{m-1}^s)=\sup_{g_i\in\mathcal D,
         i\in[1,s]}\|P_{G(s)}(r_{m-1}^s)\|,
$$
 we then write for $m\geq 2$,
\begin{eqnarray*}
           \langle r_{m-1}^s,h\rangle
           &=&
           \left\langle r_{m-1}^s,h-\sum_{j=1}^vc_jg_j\right\rangle
           =
           \left\langle r_{m-1}^s,\sum_{j=v+1}^\infty
           c_jg_j\right\rangle\\
           &=&
           \sum_{l=1}^\infty\langle r_{m-1}^s,\sum_{j\in
           J_l}c_jg_j\rangle
           =\sum_{l=1}^\infty\sum_{j\in J_l}c_j\langle
           r_{m-1}^s,g_j\rangle\\
           &\leq&
           \sum_{l=1}^\infty\left(\sum_{j\in
           J_l}c_j^2\right)^{1/2}\left(\sum_{j\in J_l}\langle
           r_{m-1}^s,g_j\rangle^2\right)^{1/2}.
\end{eqnarray*}
Hence, Lemma \ref{important estimation} implies that
\begin{eqnarray*}
       \langle r_{m-1}^s,h\rangle
       &\leq&
       \sum_{l=1}^\infty\left(\sum_{j\in
       J_l}c_j^2\right)^{1/2}(1+M(s-1))^{1/2}P_{G(J_l)}(r_{m-1})\\
       &\leq&
       \sum_{l=1}^\infty\left(\sum_{j\in
       J_l}c_j^2\right)^{1/2}(1+M(s-1))^{1/2}q_s.
\end{eqnarray*}
As the sequence $\{c_j\}$ has the property
$$
           |c_{v+1}|\geq|c_{v+2}|\geq\cdots,
           \sum_{j=v+1}^\infty|c_j|\leq\|h\|_{\mathcal L_1}+\delta,
           |c_{v+1}|\leq\frac{2(\|h\|_{\mathcal L_1}+\delta)}{s},
$$
we may apply the simple inequality
$$
           \left(\sum_{j\in J_l} c_j^2\right)^{1/2}\leq
           |c_{(l-1)s+v+1|} s^{1/2}
$$
so that we have
\begin{eqnarray*}
         \sum_{l=1}^\infty\left(\sum_{j\in J_l }c_j^2\right)^{1/2}
         &\leq&
         s^{1/2}\sum_{l=1}^\infty|c_{(l-1)s+v+1}|
         \leq
         s^{1/2}\left(\frac{2(\|h\|_{\mathcal
         L_1}+\delta)}{s}+\sum_{l=2}^\infty s^{-1}\sum_{j\in
         J_{l-1}}|c_j|\right)\\
         &\leq&
         3(\|h\|_{\mathcal
         L_1}+\delta) s^{-1/2}.
\end{eqnarray*}
Therefore, we obtain
\begin{equation}\label{second term}
         \langle r_{m-1}^s,h\rangle\leq q_s(1+M(s-1))^{1/2}3(\|h\|_{\mathcal
         L_1}+\delta)s^{-1/2}.
\end{equation}
It follows from (\ref{first term}) and (\ref{second term}) that
\begin{eqnarray*}
       \|r_{m-1}^s\|^2
       &\leq&
       \|r_{m-1}^s\|\|f-h\|+q_s(1+M(s-1))^{1/2}3(\|h\|_{\mathcal
         L_1}+\delta)s^{-1/2}\\
         &\leq&
         \frac12(\|r_{m-1}^s\|^2+\|f-h\|^2)+q_s(1+M(s-1))^{1/2}3(\|h\|_{\mathcal
         L_1}+\delta)s^{-1/2}.
\end{eqnarray*}
Denote $a_m=\|r_m^s\|^2-\|f-h\|^2$, then (\ref{second term}) implies
that
$$
          q_s\geq\frac{a_{m-1}s^{1/2}}{3(\|h\|_{\mathcal
          L_1}+\delta)(1+M(s-1))^{1/2}}.
$$

Note that if for some $k_0$, we have $\|r_{k_0-1}^s\|\leq\|f-h\|,$
then the theorem holds trivially for all $N\geq k_0-1$. We therefore
assume that $a_{k-1}$ is positive, so that we can write
$$
         q_s^2\geq\frac{a_{m-1}^2s}{9(\|h\|_{\mathcal
         L_1}+\delta)^2(1+M(s-1))}.
$$
Thus, we obtain
$$
         \|r_m^s\|^2\leq\|r_{m-1}^s\|^2-\|p_m\|^2\leq\|r_{m-1}^s\|^2-q_s^2
         \leq
         \|r_{m-1}^s\|^2-\frac{a_{m-1}^2s}{9(\|h\|_{\mathcal
         L_1}+\delta)^2(1+M(s-1))},
$$
which, by subtracting $\|f-h\|^2$ in the both sides, gives
$$
         a_m\leq a_{m-1}\left(1-\frac{a_{m-1}s}{9(\|h\|_{\mathcal
         L_1}+\delta)^2(1+M(s-1))}\right).
$$
The above inequality together with Lemma \ref{induction} yields that
\begin{equation}\label{last}
          a_m\leq\frac{9(\|h\|_{\mathcal
         L_1}+\delta)^2(1+M(s-1))}{s}m^{-1},
\end{equation}
provided that
\begin{equation}\label{intinal condition}
       a_1\leq\frac{9(\|h\|_{\mathcal
         L_1}+\delta)^2(1+M(s-1))}{s}.
\end{equation}
To prove (\ref{intinal condition}), we remark that either
$a_0\leq\frac{9(\|h\|_{\mathcal
         L_1}+\delta)^2(1+M(s-1))}{s}$, so that the same holds for
         $a_1$, or $a_0\geq \frac{9(\|h\|_{\mathcal
         L_1}+\delta)^2(1+M(s-1))}{s}$, in which case $a_1<0$ by
         (\ref{last}),
which means that we are already in the trivial case
$\|r_1^s\|\leq\|f-h\|$ for which there is noting to prove. As
$\delta$ is arbitrary positive number and $s\leq \frac1{2M}+1$, we
have
$$
        \|r_m^s\|^2\leq\|f-h\|^2+9(\|h\|_{\mathcal
         L_1}+\delta)^2(1+M(s-1))(sm)^{-1}
         \leq
         \|f-h\|^2+\frac{27}2\|h\|_{\mathcal L_1}^2(sm)^{-1}.
$$
This finishes the proof of Theorem \ref{theorem1}.
\end{proof}

To prove Theorem \ref{learning1}, we need the following lemma, which
can be found in \cite[Theorem 11.3]{Gyorfi2002}.

\begin{lemma}\label{oracle}
 Suppose that $\mathcal F_n$ is a linear vector space of
functions $f:X\rightarrow Y$ which may depend on $x_1,\dots,x_n$.
Let $\mu=\mu(x_1,\dots,x_n)$ be the vector space dimension of
$\mathcal F_n$. If we define $f_{\bf z,\mathcal F_n}$ as
$$
              f_{\bf z,\mathcal F_n}:=\arg\min_{f\in\mathcal
              F_n}\frac1n\sum_{i=1}^n|f(x_i)-y_i|^2,
$$
then we have
$$
             E(\|\Pi f_{{\bf z},\mathcal F_n}-f_\rho\|^2_\rho)\leq
             CL^2\frac{\mu\log n}{n}+8\min_{f\in\mathcal
             F_n}\|f_\rho-f\|_\rho^2,
$$
where $C$ is a universal constant.
\end{lemma}

%\begin{lemma}\label{expection}
% Given any dictionary $D$, for all function $h$ in Span$(\mathcal D)$,
%we have
%$$
%        E(\|h\|^2_{\mathcal L_1^n})\leq\|h\|_{\mathcal L_1}^2.
%$$
%\end{lemma}

Based on Lemma \ref{oracle}, we give the proof of Theorem
\ref{learning1}.

\begin{proof}[Proof of Theorem \ref{learning1}]
 We divide the OSGA(s) into two parts. The one
is to choose $sm$ atoms from the dictionary using $m$ steps and the
other is to implement a least square algorithm on an
$sm$-dimensional linear space $V_{sm}$. Once the $sm$-dimensional
linear space is fixed, it follows from Lemma \ref{oracle} that
\begin{equation}\label{le.1}
             E(\|\Pi f_{{\bf z},m}^s-f_\rho\|_\rho^2)\leq
             L^2\frac{sm\log n}{n}+ \min_{f\in
             V_{sm}}\|f_\rho-f\|_\rho^2.
\end{equation}
Thus, we only need to give an upper bound of $ \min_{f\in\mathcal
             V_{sm}}\|f_\rho-f\|_\rho^2.$
By Theorem \ref{theorem1}, we obtain for arbitrary
$h\in\mbox{span}\mathcal D_N$,
\begin{eqnarray*}
       &&\frac1n\sum_{i=1}^n(y_i-f^s_{{\bf z},m}(x_i))^2-\frac1n\sum_{i=1}^n(y_i-f_\rho(x_i))^2\\
       &\leq&
         \frac1n\sum_{i=1}^n(y_i-h(x_i))^2-\frac1n\sum_{i=1}^n(y_i-f_\rho(x_i))^2+
         \frac{27}2\|h\|_{\mathcal L_{1,N}}^2(ms)^{-1}.
\end{eqnarray*}
It follows from (\ref{equality}) that
\begin{eqnarray*}
         \|f^s_{{\bf
         z},m}-f_\rho\|_\rho^2
         &=&E\left(\frac1n\sum_{i=1}^n(y_i-f^s_{{\bf
         z},m}(x_i))^2-\frac1n\sum_{i=1}^n(y_i-f_\rho(x_i))^2\right)\\
         &=&
         E(|y-f^s_{{\bf
         z},m}(x)|^2)-E(|y-f_\rho(x)|^2)
\end{eqnarray*}
and
\begin{eqnarray*}
         \|h-f_\rho\|^2_\rho
         &=&E\left(\frac1n\sum_{i=1}^n(y_i-h(x_i))^2-\frac1n\sum_{i=1}^n(y_i-f_\rho(x_i))^2\right)\\
         &=&
         E(|y-h(x)|^2)-E(|y-f_\rho(x)|^2).
\end{eqnarray*}
The above two equalities yield that
$$
             \min_{f\in V_{sm}}\|f_\rho-f\|_\rho^2
             \leq
             \|f_\rho-f^s_{{\bf z},m}\|^2_\rho\leq \|h-f_\rho\|^2_\rho+\frac{27}{2sm}\|h\|_{\mathcal
             L_{1,N}}^2,
$$
which together with (\ref{le.1}) completes the proof of Theorem
\ref{learning1}.
\end{proof}

\section{Concluding Remarks }

The main contributions of the present paper can be summarized as
follows. Firstly, we have proposed that studying the relationship
between the ``dictionary-selection'' and ``greedy-definition''
stages can improve the learning performance of greedy learning. In
fact, we borrowed the idea of orthogonal super greedy algorithm
(OSGA) for incoherent dictionaries from nonlinear approximation and
compressive sensing \cite{Liu2012} to the supervised learning
problem and analyze the pros and cons of OSGA learning. Secondly, we
have established an approximation theorem of OSGA approximation and
show that OSGA is also available to a variety of target function
classes which is not simply related to the convex hull of the
dictionary. Such an approximation theorem is the main tool to
generalize the application of OSGA from approximation to learning.
Thirdly, we have theoretically proved that, for incoherent
dictionaries, OSGA learning can reduce the computational burden of
OGA learning without sacrificing its generalization capability.
Precisely, our error estimate for OSGA learning yields a learning
rate  as $(n/\log n)^{-1/2}$, which is the same as that of OGA
\cite{Barron2008}. Finally, we have studied the numeral performance
of OSGA. Our results show that when applied in supervised learning
problem, OSGA yields a similar prediction accuracy as both OGA and
other dictionary-based learning schemes, but has the potentials to
reduce the price in both training and test time.

To make sense of the OSGA learning presented in this paper, we
conclude this paper by the following remarks concerning some crucial
issues of OSGA learning.

\begin{remark}\label{remark1}
In Theorem \ref{learning1}, we study the learning capability of OSGA
under the assumption that the dictionary is incoherent. However, in
Section 4, we employ the simulations by utilizing two fixed
dictionaries: TPD and GRD. It is easy to see that neither TPD nor
GRD satisfies the conditions of Theorem \ref{learning1}. However,
the numerical results show that implementing OSGA in TPD can improve
the learning performance compared with the classical OGA. This fact
shows that the theoretical results about OSGA are a bit pessimistic
and the incoherence  constraint to the dictionary can be relaxed
further. Thus, we are usually  asked for an essentially  constraint
to the dictionary instead of the incoherence. Under this constraint,
OSGA can essentially improve the learning performance in the sense
that OSGA reduces the computational burden of OGA without
sacrificing its generalization capability. Admittedly, this is a
very difficult but important issue about OSGA learning.
 We will keep working on this interesting project,
and report our progress in a future publication.
\end{remark}

\begin{remark}\label{remark2}
 Practitioners have asked us frequently the following question:
 How to choose the step-size parameter $s$ for OSGA learning? This is a very good question. Admittedly, it
is often unlikely to tackle only the  dictionary with very small
coherence. Thus, judiciously choosing a value for $s$ is crucial. If
$s$ is chosen to be too large, although the training time is
reduced, the generalization capability may  be weakened. If $s$ is
selected to bee too small, then OSGA cannot essentially outperform
the classical OGA. We  think the best choice of $s$ depends heavily
on the essential condition developed in Remark \ref{remark1}. Once
the essential condition is found, the best choice of $s$ can be
consequently determined. We will also keep working on this practical
issue and report our progress in a future publication.
\end{remark}

\begin{remark}\label{remark3}
In the simulations, we present an example that in the
one-dimensional case, the  TPD dictionary perfectly guarantees the
effectiveness of OSGA learning. However, in high-dimensionally
cases, there lack  of such ``good'' and easy-implemented
dictionaries, which more or less influences the application of OSGA
learning. Therefore, how to develop generally ``good'' and
easy-implemented dictionaries for OSGA learning deserves further
studies.
\end{remark}

\end{document}